\documentclass[conference]{IEEEtran}
\IEEEoverridecommandlockouts
\usepackage[T1]{fontenc}
\usepackage{newtxtext}
\usepackage{newtxmath}
\usepackage{cite}
\usepackage{amsmath}
\usepackage{graphicx}
\usepackage{booktabs}
\usepackage{multirow}
\usepackage{xcolor}
\usepackage{tikz}
\usepackage{pgfplots}
\pgfplotsset{compat=1.18}
\usetikzlibrary{arrows.meta,positioning,calc,fit,backgrounds}
\usepackage{url}
\usepackage{textcomp}

\definecolor{slotblue}{HTML}{2A78D6}
\definecolor{slotorange}{HTML}{EB6834}
\definecolor{slotaqua}{HTML}{1BAF7A}
\definecolor{inkmuted}{HTML}{52514E}
\definecolor{gridgray}{HTML}{E4E3DF}

\newtheorem{proposition}{Proposition}
\newcommand{\method}{IMPULSE}
\definecolor{rbg}{HTML}{F0F0F0}\definecolor{rblue}{HTML}{2E67B4}\definecolor{rred}{HTML}{C8372D}
\newsavebox{\tryone}\newsavebox{\trytwo}\newsavebox{\trythree}\newsavebox{\motionbox}
\newcommand{\pct}{\%}

\begin{document}
\bstctlcite{IEEEexample:BSTcontrol}

\title{\method: Push First, Then Predict\\in V-JEPA~2 Feature Space}

\author{\IEEEauthorblockN{Tingan Jin}
\IEEEauthorblockA{University of California, Los Angeles\\
Los Angeles, CA, USA\\
annjin@ucla.edu}}

\maketitle

\begin{abstract}
Latent world models such as V-JEPA~2 forecast in the feature space of a video encoder, but one look at an object does not reveal how hard it is to rotate, how much it bounces or how much it slides. People find out by trying: a few pokes reveal how an object responds. \method{}, named after the impulse response it infers, gives a predictor in frozen V-JEPA~2 feature space the same ability. A Bayesian contact memory summarizes three short calibration pushes, a predicted object pose carries it into the predictor, and a lightweight random-Fourier pose head lets the predictor use the pose. On 512 fresh objects, in a pre-specified held-out confirmation with models sealed before the data existed, the memory lowers the 1-second feature-forecast error by 11.6--12.8\pct{} relative to the same predictor without experience, for all 8 training seeds, with a 95\pct{} seed-and-object interval excluding zero. The gain also holds for 128 objects with out-of-range physics (11.7--14.0\pct{}), and the full extension beats a plain predictor by 15.5--16.6\pct{}. An information-ladder argument, exact for ideal predictors, shows that physics reaches the features only through motion, so we measure how much the pathway lets pose help. On fresh objects, the Fourier head roughly triples the gain of a perfect memory over plain pose tokens (2.8--3.2$\times$). Raw push measurements fed through the same pathway give a similar gain, so the Bayesian summary is not required; it condenses the trials into eight interpretable numbers. The gain is absent at contact onset and grows afterwards.
\end{abstract}


\section{Introduction}

World models predict what will happen next and are a basic building block for planning agents~\cite{ha2018recurrent,hafner2025dreamerv3,lecun2022path}. Joint-embedding predictive architectures (JEPAs) learn representations by predicting the embeddings of masked or future content instead of reconstructing pixels~\cite{assran2023ijepa,bardes2024vjepa}. V-JEPA~2 couples a large self-supervised video encoder with an action-conditioned predictor trained on its frozen features and uses it for robot planning~\cite{assran2025vjepa2}, and its predecessor V-JEPA captures surprisingly much intuitive physics~\cite{garrido2025intuitive}.

\begin{figure}[t]
\centering
\definecolor{rbg}{HTML}{F0F0F0}\definecolor{rblue}{HTML}{2E67B4}\definecolor{rred}{HTML}{C8372D}\begin{tikzpicture}[font=\scriptsize]
\fill[rbg] (0.000,-2.620) rectangle (2.620,0.000);
\draw[inkmuted!60, line width=0.3pt] (0.000,-2.620) rectangle (2.620,0.000);
\path[fill=rblue, fill opacity=0.16] (1.820,-1.620) -- (1.120,-1.981) -- (1.030,-1.806) -- (1.730,-1.445) -- cycle;
\path[fill=rblue, fill opacity=0.16] (1.467,-1.580) -- (1.292,-1.671) -- (1.022,-1.145) -- (1.197,-1.055) -- cycle;
\path[fill=rblue, fill opacity=0.16] (1.694,-1.272) -- (1.183,-1.872) -- (1.033,-1.744) -- (1.545,-1.144) -- cycle;
\path[fill=rblue, fill opacity=0.16] (1.353,-1.369) -- (1.225,-1.519) -- (0.775,-1.135) -- (0.903,-0.986) -- cycle;
\path[fill=rblue, fill opacity=0.16] (1.497,-0.940) -- (1.251,-1.689) -- (1.064,-1.627) -- (1.309,-0.878) -- cycle;
\path[fill=rblue, fill opacity=0.16] (1.217,-1.159) -- (1.156,-1.346) -- (0.594,-1.162) -- (0.656,-0.975) -- cycle;
\path[fill=rblue!35!white] (1.851,-1.761) -- (1.102,-2.006) -- (1.040,-1.819) -- (1.789,-1.574) -- cycle;
\path[fill=rblue!35!white] (1.509,-1.666) -- (1.321,-1.727) -- (1.137,-1.165) -- (1.325,-1.104) -- cycle;
\path[fill=rblue] (1.239,-0.650) -- (1.296,-1.436) -- (1.099,-1.450) -- (1.042,-0.664) -- cycle;
\path[fill=rblue] (1.064,-0.959) -- (1.078,-1.155) -- (0.488,-1.198) -- (0.474,-1.002) -- cycle;
\fill[rred] (1.752,-2.049) circle (0.099);
\draw[-{Stealth[length=1.3mm]}, rred, line width=0.7pt] (1.715,-1.936) -- (1.633,-1.687);
\node[anchor=north, align=center, text=inkmuted] at (1.310,-2.660) {turns easily, rough\\$\rho{=}1200$, $e{=}0.10$, $\mu{=}0.35$};
\fill[rbg] (2.800,-2.620) rectangle (5.420,0.000);
\draw[inkmuted!60, line width=0.3pt] (2.800,-2.620) rectangle (5.420,0.000);
\path[fill=rblue, fill opacity=0.16] (4.612,-1.645) -- (3.879,-1.935) -- (3.807,-1.751) -- (4.540,-1.462) -- cycle;
\path[fill=rblue, fill opacity=0.16] (4.265,-1.571) -- (4.082,-1.643) -- (3.865,-1.093) -- (4.048,-1.021) -- cycle;
\path[fill=rblue, fill opacity=0.16] (4.511,-1.361) -- (3.827,-1.754) -- (3.729,-1.583) -- (4.413,-1.190) -- cycle;
\path[fill=rblue, fill opacity=0.16] (4.156,-1.338) -- (3.985,-1.436) -- (3.691,-0.923) -- (3.862,-0.825) -- cycle;
\path[fill=rblue, fill opacity=0.16] (4.399,-1.076) -- (3.780,-1.563) -- (3.658,-1.409) -- (4.277,-0.921) -- cycle;
\path[fill=rblue, fill opacity=0.16] (4.045,-1.104) -- (3.890,-1.226) -- (3.525,-0.761) -- (3.679,-0.640) -- cycle;
\path[fill=rblue!35!white] (4.651,-1.761) -- (3.902,-2.006) -- (3.840,-1.819) -- (4.589,-1.574) -- cycle;
\path[fill=rblue!35!white] (4.309,-1.666) -- (4.121,-1.727) -- (3.937,-1.165) -- (4.125,-1.104) -- cycle;
\path[fill=rblue] (4.277,-0.791) -- (3.735,-1.363) -- (3.592,-1.227) -- (4.134,-0.655) -- cycle;
\path[fill=rblue] (3.931,-0.870) -- (3.795,-1.013) -- (3.366,-0.606) -- (3.502,-0.463) -- cycle;
\fill[rred] (4.552,-2.049) circle (0.099);
\draw[-{Stealth[length=1.3mm]}, rred, line width=0.7pt] (4.515,-1.936) -- (4.433,-1.687);
\node[anchor=north, align=center, text=inkmuted] at (4.110,-2.660) {resists turning, rough\\$\rho{=}6000$, $e{=}0.10$, $\mu{=}0.35$};
\fill[rbg] (5.600,-2.620) rectangle (8.220,0.000);
\draw[inkmuted!60, line width=0.3pt] (5.600,-2.620) rectangle (8.220,0.000);
\path[fill=rblue, fill opacity=0.16] (7.433,-1.534) -- (6.753,-1.932) -- (6.653,-1.762) -- (7.333,-1.364) -- cycle;
\path[fill=rblue, fill opacity=0.16] (7.078,-1.513) -- (6.908,-1.613) -- (6.610,-1.103) -- (6.780,-1.003) -- cycle;
\path[fill=rblue, fill opacity=0.16] (7.297,-0.980) -- (6.904,-1.664) -- (6.734,-1.565) -- (7.126,-0.882) -- cycle;
\path[fill=rblue, fill opacity=0.16] (6.979,-1.138) -- (6.881,-1.309) -- (6.368,-1.015) -- (6.466,-0.844) -- cycle;
\path[fill=rblue, fill opacity=0.16] (7.041,-0.477) -- (7.040,-1.265) -- (6.843,-1.265) -- (6.844,-0.477) -- cycle;
\path[fill=rblue, fill opacity=0.16] (6.843,-0.773) -- (6.843,-0.970) -- (6.252,-0.969) -- (6.252,-0.772) -- cycle;
\path[fill=rblue!35!white] (7.451,-1.761) -- (6.702,-2.006) -- (6.640,-1.819) -- (7.389,-1.574) -- cycle;
\path[fill=rblue!35!white] (7.109,-1.666) -- (6.921,-1.727) -- (6.737,-1.165) -- (6.925,-1.104) -- cycle;
\path[fill=rblue] (6.706,-0.079) -- (7.098,-0.763) -- (6.927,-0.861) -- (6.535,-0.177) -- cycle;
\path[fill=rblue] (6.682,-0.433) -- (6.780,-0.604) -- (6.267,-0.898) -- (6.169,-0.727) -- cycle;
\fill[rred] (7.352,-2.049) circle (0.099);
\draw[-{Stealth[length=1.3mm]}, rred, line width=0.7pt] (7.315,-1.936) -- (7.233,-1.687);
\node[anchor=north, align=center, text=inkmuted] at (6.910,-2.660) {bouncy, slippery\\$\rho{=}3000$, $e{=}0.75$, $\mu{=}0.00$};
\end{tikzpicture}
\caption{Same look, same push, different futures. Three T-blocks that look identical receive the same push (red disk: the pusher) but differ in hidden inertia ratio $\rho$, restitution $e$ and friction $\mu$. Faint copies show the block every 0.25\,s; the solid block is its pose after 1\,s. A single frame cannot tell the three apart, but a few test pushes can. (Illustrative rollouts of the task simulator.)}
\label{fig:teaser}
\end{figure}
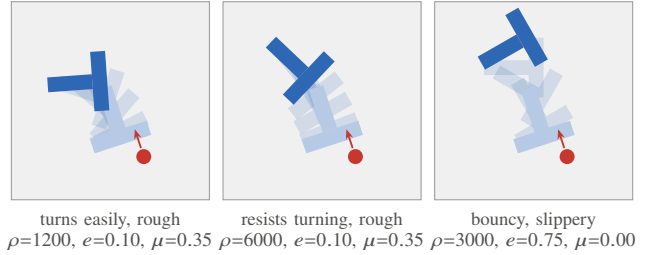

Some physical properties, however, cannot be seen. Two T-shaped blocks that look identical may differ in how their mass is distributed, how elastic they are and how much friction they meet. Pushed the same way, they move differently (Fig.~\ref{fig:teaser}), and no amount of staring at a single frame tells them apart. People solve physical problems by rapid trial and error~\cite{allen2020rapid}, and robots resolve such ambiguity by interacting: a few exploratory pokes reveal how an object responds~\cite{bohg2017interactive,zhou2019epi,xu2019densephysnet}. Physical parameters inferred from video or interaction improve physical predictions~\cite{wu2015galileo,li2025pinwm} and let control policies adapt to new dynamics~\cite{yu2017uposi,kumar2021rma}, and an embedding of a few earlier push videos of a novel object improves pixel-space video prediction~\cite{lin2020evf}. Recent work asks which physical quantities latent world models encode~\cite{tan2026latent}, supplies true physical parameters to a JEPA predictor~\cite{liu2026sgjepa}, or adapts a latent world model at test time~\cite{wang2026adajepa}. To our knowledge, how much a few active trials before the query improve a predictor whose targets are frozen, pretrained video features, and what limits that gain, has not been measured.

We study this question in a controlled pushing task with hidden inertia, restitution and friction. Our method, \method{}, pushes first and predicts second (Fig.~\ref{fig:schematic}). Before a forecast, the agent performs three short calibration pushes. A Bayesian contact memory summarizes in eight numbers how the object responds to an impulse. We compare a predictor in frozen V-JEPA~2 feature space with and without this memory, with the same architecture, data and training, and evaluate on fresh objects under a protocol fixed in advance.

With a standard pose-token pathway the gain is small: on fresh objects, even a perfect memory, holding the true physical parameters, improves the forecast by only 6.9--7.8\pct{}. The bottleneck is not the memory. Physics changes the future features only by changing where the object is, so for an ideal predictor physical knowledge can help a feature forecast at most as much as \emph{knowing the future pose}. We make this precise as an information ladder (Section~\ref{sec:ladder}) and measure every rung through the same pathway. A pose-only regressor shows that the features are highly pose-predictable, yet the pose-token predictor uses little of it. A lightweight random-Fourier pose head raises the gain of a perfect memory 2.8--3.2$\times$ and that of the real memory 2.4--4.1$\times$. Finally, a same-information control feeds the raw push measurements through the same pathway and finds a similar gain, with no clear difference from the memory. What matters is that the predictor gets to try, and that the pathway lets the result reach the features.

Our contributions are:
\begin{itemize}
\item \textbf{A push-first, predict-second world model.} \method{} extends a predictor in frozen V-JEPA~2 feature space with a contact-experience memory and a pose pathway. In a pre-specified held-out confirmation on 512 fresh objects, they cut 1-s feature-forecast error by 11.6--12.8\pct{} for all 8 training seeds. The gain holds for out-of-range physics and for frictionless objects, and the full extension beats a plain predictor by 15.5--16.6\pct{}. A same-information control shows that the gain does not require the Bayesian form of the summary.
\item \textbf{An information ladder.} A simple proposition orders the errors of predictors that know nothing, a memory, the true physical parameters, or the true future pose. The ladder turns ``does the memory help?'' into two measurable questions: how much the pathway lets pose help, and what share of that the memory captures.
\item \textbf{The pathway is the lever.} On fresh objects, a random-Fourier pose head multiplies the absolute gain of a perfect memory by 2.8--3.2, with 95\pct{} bootstrap intervals that stay above 2.3, and that of the real memory by 2.4--4.1. The improvement is absent at contact onset and grows afterwards.
\end{itemize}
For semantic computing, the study shows how a compact, interpretable description of an object's physical response can be injected into a learned semantic feature space, and how to budget what such knowledge can buy.

\begin{figure*}[t]
\centering
\sbox{\tryone}{\definecolor{rbg}{HTML}{F0F0F0}\definecolor{rblue}{HTML}{2E67B4}\definecolor{rred}{HTML}{C8372D}\begin{tikzpicture}
\fill[rbg] (0.000,-1.050) rectangle (1.050,0.000);
\draw[inkmuted!60, line width=0.3pt] (0.000,-1.050) rectangle (1.050,0.000);
\path[fill=rblue!35!white] (0.884,-0.306) -- (0.685,-0.848) -- (0.549,-0.798) -- (0.749,-0.257) -- cycle;
\path[fill=rblue!35!white] (0.674,-0.460) -- (0.624,-0.595) -- (0.218,-0.445) -- (0.268,-0.310) -- cycle;
\path[fill=rblue] (0.681,-0.234) -- (0.554,-0.797) -- (0.413,-0.765) -- (0.540,-0.202) -- cycle;
\path[fill=rblue] (0.492,-0.413) -- (0.461,-0.554) -- (0.038,-0.459) -- (0.070,-0.318) -- cycle;
\fill[rred] (0.915,-0.482) circle (0.072);
\end{tikzpicture}}\sbox{\trytwo}{\definecolor{rbg}{HTML}{F0F0F0}\definecolor{rblue}{HTML}{2E67B4}\definecolor{rred}{HTML}{C8372D}\begin{tikzpicture}
\fill[rbg] (0.000,-1.050) rectangle (1.050,0.000);
\draw[inkmuted!60, line width=0.3pt] (0.000,-1.050) rectangle (1.050,0.000);
\path[fill=rblue!35!white] (1.001,-0.295) -- (0.683,-0.933) -- (0.524,-0.854) -- (0.841,-0.216) -- cycle;
\path[fill=rblue!35!white] (0.722,-0.455) -- (0.643,-0.615) -- (0.165,-0.377) -- (0.244,-0.217) -- cycle;
\path[fill=rblue] (0.888,-0.244) -- (0.474,-0.824) -- (0.329,-0.720) -- (0.743,-0.140) -- cycle;
\path[fill=rblue] (0.588,-0.358) -- (0.484,-0.503) -- (0.049,-0.192) -- (0.153,-0.047) -- cycle;
\fill[rred] (0.832,-0.884) circle (0.089);
\end{tikzpicture}}\sbox{\trythree}{\definecolor{rbg}{HTML}{F0F0F0}\definecolor{rblue}{HTML}{2E67B4}\definecolor{rred}{HTML}{C8372D}\begin{tikzpicture}
\fill[rbg] (0.000,-1.050) rectangle (1.050,0.000);
\draw[inkmuted!60, line width=0.3pt] (0.000,-1.050) rectangle (1.050,0.000);
\path[fill=rblue!35!white] (0.943,-0.784) -- (0.398,-0.997) -- (0.345,-0.861) -- (0.889,-0.648) -- cycle;
\path[fill=rblue!35!white] (0.685,-0.727) -- (0.549,-0.781) -- (0.389,-0.372) -- (0.526,-0.319) -- cycle;
\path[fill=rblue] (0.756,-0.349) -- (0.305,-0.722) -- (0.212,-0.610) -- (0.662,-0.237) -- cycle;
\path[fill=rblue] (0.494,-0.377) -- (0.381,-0.470) -- (0.101,-0.132) -- (0.214,-0.039) -- cycle;
\fill[rred] (0.851,-0.914) circle (0.073);
\end{tikzpicture}}\sbox{\motionbox}{\definecolor{rbg}{HTML}{F0F0F0}\definecolor{rblue}{HTML}{2E67B4}\definecolor{rred}{HTML}{C8372D}\begin{tikzpicture}
\fill[rbg] (0.000,-1.250) rectangle (1.250,0.000);
\draw[inkmuted!60, line width=0.3pt] (0.000,-1.250) rectangle (1.250,0.000);
\path[fill=rblue!35!white] (0.871,-0.199) -- (1.210,-0.689) -- (1.088,-0.774) -- (0.748,-0.284) -- cycle;
\path[fill=rblue!35!white] (0.876,-0.467) -- (0.961,-0.590) -- (0.593,-0.844) -- (0.508,-0.722) -- cycle;
\draw[slotblue, line width=0.6pt, densely dashed] (0.776,-0.253) -- (1.051,-0.782) -- (0.919,-0.851) -- (0.816,-0.652) -- (0.420,-0.859) -- (0.351,-0.727) -- (0.747,-0.520) -- (0.643,-0.322) -- cycle;
\draw[slotblue, line width=0.9pt] (0.593,-0.439) -- (0.687,-1.027) -- (0.540,-1.051) -- (0.505,-0.830) -- (0.063,-0.901) -- (0.040,-0.754) -- (0.481,-0.683) -- (0.446,-0.463) -- cycle;
\end{tikzpicture}}%
\begin{tikzpicture}[font=\footnotesize,
  box/.style={draw=inkmuted, rounded corners=2pt, align=center, inner sep=3pt, fill=white, minimum height=6.5mm},
  ours/.style={box, draw=slotblue, line width=0.9pt},
  arr/.style={-{Stealth[length=1.8mm]}, inkmuted, line width=0.7pt},
  stepl/.style={font=\footnotesize\bfseries, text=inkmuted, anchor=south west, inner sep=0}]
\node[anchor=south west, inner sep=0] (t1) at (0,0) {\usebox{\tryone}};
\node[anchor=south west, inner sep=0] (t2) at (1.2,0) {\usebox{\trytwo}};
\node[anchor=south west, inner sep=0] (t3) at (2.4,0) {\usebox{\trythree}};
\node[stepl] at (0,1.55) {1. Push first};
\node[font=\scriptsize, text=inkmuted, anchor=north west, align=left, inner sep=0] at (0,-0.12) {three test pushes; only noisy\\pose tracks are observed};
\node[anchor=south west, inner sep=0] (mem) at (4.25,-0.62) {\begin{tikzpicture}
\begin{axis}[width=5.0cm, height=3.0cm, ybar, bar width=2.1pt, ymin=-1.7, ymax=2.7,
  xtick={0,...,7}, xticklabels={$v_n$,$v_t$,$\omega$,$v_n$,$v_t$,$\omega$,$\sigma$,$\sigma$},
  ytick={-1,0,1}, tick label style={font=\scriptsize, text=inkmuted}, axis line style={inkmuted!70}, tick style={inkmuted!70},
  axis x line*=bottom, axis y line*=left, ymajorgrids, grid style={gridgray}, enlarge x limits=0.08,
  legend style={font=\scriptsize, draw=none, fill=white, fill opacity=0.85, text opacity=1, at={(0.5,0.99)}, anchor=north, legend columns=2, inner sep=1pt}]
\addplot[fill=slotblue, draw=none] table[x=slot, y=rough] {data_memory.dat};
\addlegendentry{rough}
\addplot[fill=slotaqua, draw=none] table[x=slot, y=slippery] {data_memory.dat};
\addlegendentry{frictionless}
\end{axis}\end{tikzpicture}};
\node[stepl] at (4.35,1.55) {2. Contact memory $m$ (8 numbers)};
\node[anchor=south west, inner sep=0] (mot) at (9.55,0) {\usebox{\motionbox}};
\node[stepl] at (9.55,1.55) {3. Motion head};
\node[font=\scriptsize, text=inkmuted, anchor=north west, align=left, inner sep=0] at (9.55,-0.12) {predicted pose $\hat p(t)$\\at 0.5\,s and 1\,s};
\node[box, minimum width=31mm, font=\scriptsize] (query) at (13.65,1.55) {query frames, pose, pusher, actions};
\node[box, minimum width=31mm] (pred) at (13.65,0.6) {V-JEPA 2 predictor};
\node[ours, minimum width=31mm] (four) at (13.65,-0.35) {Fourier pose head};
\node[stepl] at (12.1,-1.05) {4. Predict};
\node[box, text width=19mm] (out) at (17.05,0.15) {future V-JEPA~2\\features $\hat y(t)$};
\draw[arr] (t3.east) -- (mem.west |- t3.east);
\draw[arr] (mem.east |- t3.east) -- (mot.west);
\draw[arr, slotblue] (mot.east) -- node[above, font=\scriptsize, text=slotblue, pos=0.45]{tokens} (pred.west);
\draw[arr, slotblue] (mot.east) -- (four.west);
\draw[arr] (query.south) -- (pred.north);
\draw[arr] (pred.east) -- (out.west |- pred.east);
\draw[arr, slotblue] (four.east) -- node[below, font=\scriptsize, text=slotblue, pos=0.5]{$+$} (out.west |- four.east);
\end{tikzpicture}
\caption{\method{} on a simulated confirmation object. (1)~Three test pushes (light: start, dark: end; red: pusher). (2)~The Bayesian contact memory: posterior mean normal, tangential and angular response at two lever arms and two uncertainties, shown for the rough ($\mu=0.35$) and the frictionless version of the same object. The tangential responses $v_t$ reveal friction. (3)~A motion head turns the memory and the query context into predicted poses (outlines). (4)~The pose enters the V-JEPA~2 predictor as tokens and through a Fourier pose head that corrects the predicted features. The query context also feeds the predictor. Blue marks our additions; the no-memory arm has the same architecture with $m=0$.}
\label{fig:schematic}
\end{figure*}
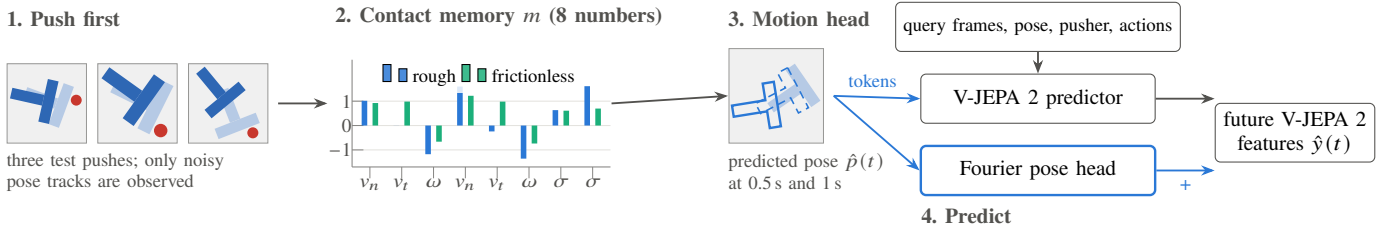

\section{Related Work}

\textbf{Latent world models.} Recurrent and imagination-based world models learn compact dynamics for control~\cite{ha2018recurrent,hafner2025dreamerv3}. JEPAs predict masked or future content in the representation space of an encoder~\cite{lecun2022path,assran2023ijepa,bardes2024vjepa}. V-JEPA~2 adds an action-conditioned predictor (V-JEPA~2-AC) on top of a frozen video encoder~\cite{assran2025vjepa2}, and DINO-WM plans with a world model over frozen DINOv2 features, including on Push-T~\cite{zhou2025dinowm}. These models predict from what they see. Contextual world models condition dynamics on environment parameters~\cite{prasanna2024crssm}; Semigroup-JEPA gives a JEPA predictor the true physical parameter~\cite{liu2026sgjepa}, and AdaJEPA adapts a latent world model at test time from observed transitions~\cite{wang2026adajepa}. PokeWorld asks which hidden quantities (mass, drag, contact stiffness) the latent states of action-conditioned world models encode when visually identical objects are pushed~\cite{tan2026latent}. We instead infer the parameters from a few trials before the forecast, feed them to a predictor in a frozen video-encoder feature space, and measure what limits the gain.

\textbf{Physical properties from video and interaction.} Physics engines combined with learning infer mass and friction from videos~\cite{wu2015galileo,wu2016physics101}. Interaction networks learn object dynamics~\cite{battaglia2016interaction}, and benchmarks probe physical prediction from vision~\cite{bear2021physion,yi2020clevrer}. ComPhy~\cite{chen2022comphy} and Physion++~\cite{tung2023physionpp} require inferring hidden properties such as mass, friction or elasticity from how objects interact. Features of video foundation models, V-JEPA~2 among them, encode dynamic properties such as elasticity and friction when the motion is observed~\cite{zhan2026idpp}. PIN-WM identifies friction and restitution from a few interactions with a differentiable physics simulator~\cite{li2025pinwm}. We infer the properties from a few active trials \emph{before} the forecast, and our predictor has no simulator: it forecasts encoder features. Interactive perception argues that acting reveals what passive perception cannot~\cite{bohg2017interactive}. DensePhysNet~\cite{xu2019densephysnet} learns physical representations from active interactions, and environment-probing policies~\cite{zhou2019epi} explicitly probe before acting. Planar pushing is a classical testbed for such questions~\cite{mason1986pushing,yu2016pushing,bauza2017probabilistic,ajay2018augmenting}, and Push-T is a standard visuomotor benchmark~\cite{florence2022ibc,chi2023diffusionpolicy}.

\textbf{Adaptation and system identification.} Hidden-parameter MDPs formalize dynamics that differ by a few latent parameters~\cite{doshivelez2016hipmdp}. Policies conditioned on identified or latent physical parameters adapt to new dynamics~\cite{yu2017uposi,kumar2021rma}, learned dynamics models can be adapted online from recent experience~\cite{nagabandi2019adapt}, and SimDist keeps a simulation-pretrained encoder fixed while adapting only its latent dynamics to real-world data~\cite{levy2026simdist}. Neural processes and set encoders summarize a context set into a conditioning vector~\cite{garnelo2018cnp,zaheer2017deepsets}. Closest to our work, experience-embedded visual foresight encodes five earlier push videos of a novel object into such a vector and conditions a pixel-space video predictor on it, evaluated on real pushes of unseen objects~\cite{lin2020evf,bauza2019omnipush}. Our memory plays the same role, but its targets are frozen video-encoder features, its gain is confirmed on held-out objects under a protocol fixed in advance, and we show that the pathway from memory to those features decides the gain.

\textbf{Fourier features.} Random Fourier features approximate shift-invariant kernels with a linear readout~\cite{rahimi2007rff}, and Fourier input mappings let networks fit high-frequency functions of low-dimensional coordinates~\cite{tancik2020fourier}. We use them to map object pose to its appearance in V-JEPA~2 feature space.

\section{Setting and Method}

\subsection{Task: forecasting pushes of objects with hidden physics}

A T-shaped block moves on a plane in a two-dimensional rigid-body simulator with Coulomb contact (no gravity and no damping, so the block glides between contacts; 2.5-ms steps). Its shape is known; three properties are hidden: the ratio $\rho$ of moment of inertia to mass, the restitution $e$ and the friction coefficient $\mu$. For in-distribution (ID) objects, $\rho$ is log-uniform in $[1200,6000]$\,px$^2$ and $e$ is uniform in $[0.05,0.75]$. Each object appears twice, with $\mu=0$ and $\mu=0.35$, under identical pushes. Out-of-distribution (OOD) objects have low ($[650,1000]$) or high ($[7500,10500]$) inertia ratio, or high restitution ($[0.85,0.95]$), all outside the training ranges. The simulator sets $\rho$ independently of the drawn shape. Mass confined to the T cannot give $\rho>5850$\,px$^2$, the squared distance of its farthest point from the center of mass, so the uppermost training values and all high-inertia OOD objects are stress tests rather than physically realizable T blocks.

\textbf{Experience.} Before the query, a pusher performs three calibration pushes at lever arms of 25--45\,px from the center line, at random orientations and speeds. The agent observes only noisy pose tracks (position noise 0.5--2\,px, angle noise 0.01--0.04\,rad).

\textbf{Query and target.} The query is an off-center push of the same object from a new pose. The predictor sees two rendered frames before the push ($-0.1$\,s and $0$\,s, $256\times256$ pixels), the noisy initial pose, the pusher position and the planned push velocities. The targets are the features of a frozen V-JEPA~2 ViT-L/16 encoder~\cite{assran2025vjepa2} (layer 24, $2\times2$ spatial pooling, $4\times1024=4096$ dimensions) for the two-frame clips that end 0.2, 0.5 and 1.0\,s after the push starts. The renderer draws only the T shape, in one fixed color, and the pusher, so physics affects the targets only through motion. The pusher first touches the object after a median of 0.18\,s (95th percentile 0.30\,s). At 0.2\,s contact has therefore at most just begun, whereas every push has made contact by 0.35\,s.

\textbf{Data.} Training uses 2,048 ID objects (4,096 pushes) and checkpoint selection uses 256 validation objects. The confirmation set has 512 fresh ID objects and 128 fresh OOD objects, all with both friction values. It was generated only after all models were sealed.

\subsection{The predictor}
\label{sec:predictor}

We adapt the official V-JEPA~2-AC predictor code into a compact action-conditioned transformer (width 128, two blocks, four heads, causal attention over frame groups; 845,056 parameters). It is not the released, much larger V-JEPA~2-AC model: we train this compact predictor from scratch on frozen V-JEPA~2 features to minimize the mean squared error of the three standardized future feature maps. The baseline sees only the query: its observed frames, noisy initial pose, pusher position and actions. It has no access to the calibration pushes. This isolates the value of trying.

\subsection{Contact-experience memory}

From each completed calibration push, a linear fit to the noisy track gives the post-contact normal, tangential and angular velocity, and its residuals estimate the observation noise. A grid posterior over $(\log\rho,e,\mu)$ ($161\times129\times25$ points; uniform priors on $\log\rho$ and $e$; half of the friction prior mass at $\mu=0$) is computed from these responses with a closed-form planar impulse model with Coulomb stick and slip, and 16 particles are drawn. The grid spans $\rho\in[600,12000]$\,px$^2$, $e\in[0,0.98]$ and $\mu\in[0,0.6]$, so it contains the OOD parameters even though the predictor never saw them in training. The memory thus uses the simulator's contact-model class and a friction prior matched to our two friction values; we discuss this in Section~\ref{sec:limits}. The memory $m$ reports, for two reference lever arms (25 and 45\,px), the posterior mean response (normal, tangential and angular velocity) and its spread: $3\times2+2=8$ numbers. The memory is deterministic, has no learned parameters and never sees the query.

\subsection{From memory to features: the pose pathway}

A small motion head (7,299 parameters) combines the memory and the query context multiplicatively and predicts the object's body-frame displacement and heading at six times up to 1\,s. A pose loss supervises it. From the predicted poses at $t-0.1$ and $t$, we form $\hat p(t)=(x,y,\sin\theta,\cos\theta,v_x,v_y,\omega)$ for each target time. The pose reaches the predictor in two ways:
\begin{enumerate}
\item \textbf{Pose tokens (P1).} A linear embedding of $\hat p(t)$ (1,024 parameters) is inserted at the target-time positions of the predictor.
\item \textbf{Fourier pose head (P2 = P1 + head).} Let $z$ collect the object pose at $t$ and $t-0.1$ (position divided by the 512-px arena size, heading as sine and cosine) and the pusher position at those times; the pusher position follows from the known actions. With a fixed random matrix $W\in\mathbb{R}^{1024\times12}$ ($W_{ij}\sim\mathcal N(0,9)$) and phases $b\sim\mathcal U[0,2\pi)$, the head adds $r(t)=U\,V\,[z;\cos(Wz+b)]$ to the predicted features, where the affine maps $V$ and $U$ (with biases) form a rank-64 bottleneck (332,608 parameters) and $U$ starts at zero.
\end{enumerate}
Both arms use this architecture (1,185,987 parameters); the \emph{no-memory arm} sets $m=0$. The \emph{plain predictor} of Section~\ref{sec:predictor}, without motion head and pose pathway, serves as a secondary reference. Training minimizes the feature loss plus $0.5\times$(translation + orientation) pose loss, with AdamW~\cite{loshchilov2019adamw} (learning rate $3\cdot10^{-4}$, weight decay 0.01, batch 16, 12,000 updates, gradient clipping 1).

\subsection{What can physical knowledge buy? An information ladder}
\label{sec:ladder}

Let $C$ be what the baseline sees, $E$ the calibration experience, $m=g(E)$ any memory, $\vartheta=(\rho,e,\mu)$ the true parameters, $P$ the true object poses in the target clips and $Y$ the target features. Write $\epsilon(Z)=\mathbb E\|Y-\mathbb E[Y\mid Z]\|^2$ for the error of the best predictor given information $Z$.

\begin{proposition}[Information ladder]
Assume (i) the experience depends on the query only through the object, $E\perp(Y,P)\mid(C,\vartheta)$ (the query pose and actions are drawn independently of the calibration outcomes), and (ii) the targets depend on the physics only through motion, $Y\perp(\vartheta,E)\mid(C,P)$. Then for every memory $m=g(E)$
\begin{equation}
\epsilon(C)\;\ge\;\epsilon(C,m)\;\ge\;\epsilon(C,\vartheta)\;\ge\;\epsilon(C,P).
\label{eq:ladder}
\end{equation}
\end{proposition}
\emph{Proof sketch.} Conditioning on more information never increases the squared error of the conditional mean. Hence $\epsilon(C,m)\ge\epsilon(C,E)\ge\epsilon(C,E,\vartheta)$, and the last term equals $\epsilon(C,\vartheta)$ by (i). Likewise $\epsilon(C,\vartheta)\ge\epsilon(C,\vartheta,P)=\epsilon(C,P)$ by (ii). \hfill$\square$

Assumption (ii) holds by construction, because the renderer reads only geometry. Assumption (i) holds up to one shared variable: each object's observation-noise level also sets the noise of the query pose, so the experience can reveal how reliable the observed initial pose is. Adding this level to $\vartheta$ restores (i); our oracle arm feeds only $(\rho,e,\mu)$, so it is a reference for the physical parameters, not for everything the experience carries. The proposition orders the errors of \emph{ideal} predictors; in a deterministic simulator the lower rungs can be small. Its value here is a measurement recipe. The ladder implies a \emph{gain budget}: whatever the memory, its gain over the baseline is at most the gain of the true parameters, which in turn is at most the gain of knowing the future pose. A trained architecture realizes only part of each rung, and its measured rungs need not follow the ideal ordering, so we measure them by feeding $m$, $\vartheta$ (``oracle memory'') and $P$ (``true pose'') through the \emph{same} pathway. The accounting identity
\begin{equation}
\Delta_m=\underbrace{\frac{\Delta_m}{\Delta_\vartheta}}_{\text{memory quality}}\times\underbrace{\Delta_\vartheta}_{\text{set by the pathway}}
\end{equation}
separates two levers, without implying that they are independent. Improving the memory raises the first factor; improving the pathway raises the second, which mattered more in our case.

Why should a Fourier head help? The renderer and encoder make the features a smooth but highly nonlinear function of the object's position and heading. A linear pose embedding feeding a two-block predictor struggles to represent such a map from finite data, while random Fourier features form a fixed nonlinear basis that approximates a Gaussian kernel~\cite{rahimi2007rff}, a standard remedy for this kind of low-dimensional, high-frequency mapping~\cite{tancik2020fourier}. Indeed, a ridge regression on Fourier features of the true future object pose and the known pusher position, using no video at all, already reaches a 1-s error of 0.218, far below the plain predictor's 0.411 (development set).

\section{Evaluation Protocol}

\textbf{Metric.} We report the mean squared error over the 4096 feature dimensions, in units standardized by training statistics (NMSE; lower is better). The primary endpoint is the 1-s error on $\mu=0.35$ pushes of ID objects.

\textbf{Per-model calibration.} A model can look better merely by predicting values closer to the mean. We therefore shrink every model toward the training mean, $\bar y+\alpha(\hat y-\bar y)$. Each model, checkpoint and horizon gets its own $\alpha\in[0,1]$. It is fit on one half of the evaluation objects (split by object-seed parity, keeping both friction versions of an object together) and applied to the other half, so every object is scored once; maps are never shared across models. Uncalibrated errors are reported too.

\textbf{Fairness.} Both arms share the architecture, initialization, batch order, data, budget, pose loss and selection rule. We report two checkpoints: the earliest of the checkpoints saved every 200 updates that attains the minimum validation feature loss (``selected''), and the end of training after 12,000 updates (``final''). A gain must hold at both, so a predictor that is still undertrained cannot win.

\textbf{Development and confirmation.} Design choices were made on the validation set using only oracle diagnostics (true pose, true parameters). The memory comparison itself was examined twice there. The confirmation was then specified in advance. The protocol file was frozen and hashed (SHA-256 prefix \texttt{bdaaeafd}); 8 new seeds per model were trained, and all 48 checkpoints were frozen and hashed before the confirmation objects were generated; every object was scored once and none was removed. The fresh objects come from the same generator (ID) and from a shifted-physics cohort (OOD), so the confirmation tests generalization to new objects and training runs, not to new environments. The rule requires that, at both checkpoints, the 95\pct{} interval of the mean difference lie below zero and at least 6 of 8 seeds improve. The interval treats seeds \emph{and} objects as random. For the seed-by-object matrix $D$ of paired differences, we use the conservative two-way variance estimate $\hat V=(\mathrm{MS}_{\text{seed}}+\mathrm{MS}_{\text{object}})/(SO)$ with Satterthwaite degrees of freedom~\cite{satterthwaite1946}, so the claim generalizes to new training runs and new objects. Secondary rows are not corrected for multiplicity.

\textbf{Post hoc analyses.} After the confirmation, we measured the information ladder and a same-information control on the same fresh objects, and re-analyzed the calibration and the OOD cohort. The 56 additional models (seven arms, 8 seeds each) used the sealed recipe and seeds and were hashed before scoring. These analyses are labelled post hoc; no model or setting was chosen on the confirmation objects, and the confirmatory result is unchanged.

\section{Results}

\subsection{Confirmation: the memory improves the forecast on fresh objects}

\begin{table*}[t]
\caption{Pre-specified held-out confirmation on fresh objects (1-s feature NMSE unless noted; lower is better). $\Delta$ is ours minus the comparison with a 95\pct{} interval over seeds and objects; ``seeds'' counts seeds that improved. The post hoc rows were added after the confirmation; their OOD groups are small and descriptive.}
\label{tab:confirm}
\centering
\setlength{\tabcolsep}{4.5pt}
\small
\begin{tabular}{@{}llccccc@{}}
\toprule
Comparison & checkpoint & comparison model & ours & $\Delta$ [95\pct{} CI over seeds and objects] & relative & seeds improved\\
\midrule
\multicolumn{7}{@{}l}{\emph{Primary: memory vs.\ no memory, ID, $\mu=0.35$}}\\
 & sel. & 0.416 & 0.367 & \ensuremath{-}0.048 [\ensuremath{-}0.060, \ensuremath{-}0.037] & \ensuremath{-}11.6\pct & 8/8\\
 & final & 0.417 & 0.364 & \ensuremath{-}0.053 [\ensuremath{-}0.064, \ensuremath{-}0.042] & \ensuremath{-}12.8\pct & 8/8\\
\midrule
\multicolumn{7}{@{}l}{\emph{Secondary (pre-declared)}}\\
vs.\ plain predictor & sel. & 0.435 & 0.367 & \ensuremath{-}0.067 [\ensuremath{-}0.078, \ensuremath{-}0.057] & \ensuremath{-}15.5\pct & 8/8\\
 & final & 0.436 & 0.364 & \ensuremath{-}0.073 [\ensuremath{-}0.084, \ensuremath{-}0.061] & \ensuremath{-}16.6\pct & 8/8\\
OOD objects & sel. & 0.459 & 0.405 & \ensuremath{-}0.054 [\ensuremath{-}0.070, \ensuremath{-}0.037] & \ensuremath{-}11.7\pct & 8/8\\
 & final & 0.469 & 0.403 & \ensuremath{-}0.066 [\ensuremath{-}0.084, \ensuremath{-}0.047] & \ensuremath{-}14.0\pct & 8/8\\
Frictionless ($\mu=0$) & sel. & 0.431 & 0.398 & \ensuremath{-}0.034 [\ensuremath{-}0.048, \ensuremath{-}0.019] & \ensuremath{-}7.8\pct & 8/8\\
Horizon 0.5\,s & sel. & 0.297 & 0.271 & \ensuremath{-}0.026 [\ensuremath{-}0.032, \ensuremath{-}0.021] & \ensuremath{-}8.8\pct & 8/8\\
Horizon 0.2\,s (contact onset) & sel. & 0.329 & 0.329 & \ensuremath{-}0.001 [\ensuremath{-}0.015, 0.014] & \ensuremath{-}0.2\pct & 4/8\\
Uncalibrated & sel. & 0.434 & 0.387 & \ensuremath{-}0.047 [\ensuremath{-}0.061, \ensuremath{-}0.033] & \ensuremath{-}10.8\pct & 8/8\\
Position error (px) & sel. & 27.5 & 14.2 & \ensuremath{-}13.3 [\ensuremath{-}15.2, \ensuremath{-}11.4] & \ensuremath{-}48.5\pct & 8/8\\
Heading error (rad) & sel. & 0.252 & 0.153 & \ensuremath{-}0.099 [\ensuremath{-}0.126, \ensuremath{-}0.073] & \ensuremath{-}39.4\pct & 8/8\\
\midrule
\multicolumn{7}{@{}l}{\emph{Post hoc, on the same fresh objects}}\\
vs.\ raw measurements & sel. & 0.364 & 0.367 & +0.003 [\ensuremath{-}0.005, 0.012] & +1.0\pct & 3/8\\
(same information) & final & 0.363 & 0.364 & +0.001 [\ensuremath{-}0.008, 0.010] & +0.3\pct & 4/8\\
$\alpha$ fit on validation set & sel. & 0.416 & 0.367 & \ensuremath{-}0.048 [\ensuremath{-}0.060, \ensuremath{-}0.037] & \ensuremath{-}11.6\pct & 8/8\\
 & final & 0.417 & 0.364 & \ensuremath{-}0.053 [\ensuremath{-}0.064, \ensuremath{-}0.042] & \ensuremath{-}12.7\pct & 8/8\\
OOD, high inertia ($n=32$) & sel. & 0.433 & 0.333 & \ensuremath{-}0.100 [\ensuremath{-}0.128, \ensuremath{-}0.071] & \ensuremath{-}23.0\pct & 8/8\\
OOD, high restitution ($n=64$) & sel. & 0.454 & 0.406 & \ensuremath{-}0.048 [\ensuremath{-}0.072, \ensuremath{-}0.024] & \ensuremath{-}10.5\pct & 8/8\\
OOD, low inertia ($n=32$) & sel. & 0.495 & 0.475 & \ensuremath{-}0.020 [\ensuremath{-}0.053, 0.012] & \ensuremath{-}4.1\pct & 8/8\\
\bottomrule
\end{tabular}
\end{table*}

Table~\ref{tab:confirm} shows the confirmation. The memory lowers the primary error from 0.416 to 0.367 at the selected checkpoint and from 0.417 to 0.364 at the end of training (\ensuremath{-}11.6\pct{} and \ensuremath{-}12.8\pct{}). All 8 seeds improve at both checkpoints, and both intervals lie well below zero, so the pre-specified rule is met. Fig.~\ref{fig:seeds} shows that the effect is consistent seed by seed. The comparison with the plain predictor measures the whole extension (memory, motion head and Fourier pose head) and gives \ensuremath{-}15.5\pct{} and \ensuremath{-}16.6\pct{}; at the selected checkpoint about 4.3 of these points come from the pose pathway without memory and 11.1 from the memory. The gain also transfers to objects whose inertia or restitution lies outside the training range (\ensuremath{-}11.7\pct{} and \ensuremath{-}14.0\pct{}), and it costs nothing on frictionless objects, where the error also drops by 7.8\pct{}. By shift type (post hoc, 32--64 objects per group), the gain is largest for high inertia (\ensuremath{-}23.0\pct{}), clear for high restitution (\ensuremath{-}10.5\pct{}) and smallest for low inertia (\ensuremath{-}4.1\pct{} and \ensuremath{-}6.9\pct{}), whose interval includes zero at both checkpoints. Fitting each model's $\alpha$ on its validation predictions, instead of cross-fitting it on the confirmation objects, leaves the primary difference unchanged (\ensuremath{-}0.048 and \ensuremath{-}0.053; post hoc). The memory also cuts the predicted 1-s position and heading errors by 48.5\pct{} and 39.4\pct{}. This is the physical improvement that the features inherit. On the development objects the no-memory arm scored 0.397 rather than 0.416, yet the memory's relative gain was nearly the same (\ensuremath{-}11.4\pct{}). Fig.~\ref{fig:qual} shows the first four confirmation objects in seed order, without selection. Each object appears twice under the same push, once frictionless and once rough. The no-memory predictor receives identical inputs for the two versions and must predict a single pose for both. \method{} separates them, and its position error is lower in 6 of the 8 cases.

\begin{figure}[t]
\centering
\begin{tikzpicture}
\begin{axis}[
  width=0.97\columnwidth, height=44mm,
  xmin=-0.6, xmax=3.6, ymin=-0.105, ymax=0.012,
  xtick={0,1,2,3},
  xticklabels={{ID, sel.},{ID, final},{OOD, sel.},{OOD, final}},
  ylabel={$\Delta$ 1-s NMSE (ours $-$ no memory)},
  scaled y ticks=false, yticklabel style={/pgf/number format/fixed, /pgf/number format/precision=2},
  ytick={-0.1,-0.08,-0.06,-0.04,-0.02,0},
  ylabel style={font=\footnotesize, text=inkmuted},
  tick label style={font=\footnotesize, text=inkmuted},
  axis line style={inkmuted}, tick style={inkmuted},
  ymajorgrids, grid style={gridgray},
  axis x line*=bottom, axis y line*=left,
  clip=false]
\addplot[inkmuted, dashed, line width=0.6pt, domain=-0.6:3.6, samples=2] {0};
\addplot[only marks, mark=*, mark size=1.7pt, color=slotblue, fill opacity=0.7, draw opacity=0.9]
  table[x expr=\thisrow{cell}-0.18+0.05*\thisrow{seed}, y=diff] {data_seeds.dat};
\addplot[only marks, mark=-, mark size=6pt, line width=1.1pt, color=black,
  error bars/.cd, y dir=both, y explicit, error bar style={line width=1.1pt, black}, error mark=]
  table[x expr=\thisrow{cell}+0.32, y=mean, y error plus expr=\thisrow{hi}-\thisrow{mean}, y error minus expr=\thisrow{mean}-\thisrow{lo}] {data_seeds_ci.dat};
\node[font=\scriptsize, text=inkmuted, anchor=south west] at (axis cs:-0.55,0.001) {no effect};
\end{axis}
\end{tikzpicture}
\caption{Confirmation, seed by seed. Blue dots are the 8 training seeds (each averaged over all fresh objects); black bars show the mean with its 95\pct{} interval over seeds and objects. Every seed lies below zero (lower error with memory) in every cell.}
\label{fig:seeds}
\end{figure}
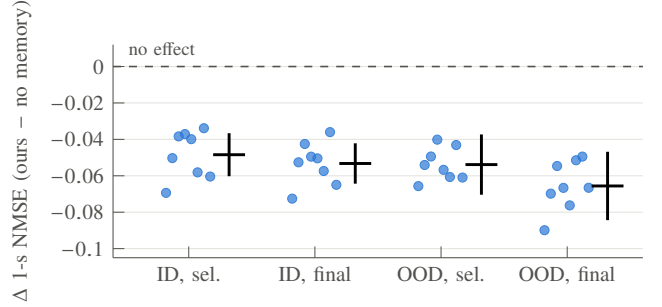

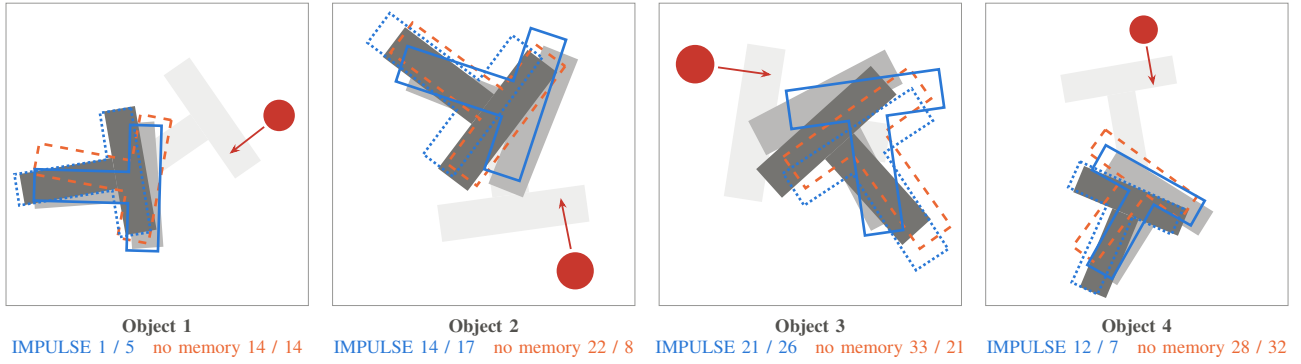
\begin{figure*}[t]
\centering
\definecolor{rbg}{HTML}{F0F0F0}\definecolor{rblue}{HTML}{2E67B4}\definecolor{rred}{HTML}{C8372D}\begin{tikzpicture}[font=\scriptsize]
\fill[white] (0.000,-4.000) rectangle (4.000,0.000);
\draw[inkmuted!60, line width=0.3pt] (0.000,-4.000) rectangle (4.000,0.000);
\path[fill=inkmuted!10] (2.427,-0.718) -- (3.372,-2.083) -- (3.031,-2.320) -- (2.085,-0.955) -- cycle;
\path[fill=inkmuted!10] (2.440,-1.467) -- (2.676,-1.808) -- (1.652,-2.517) -- (1.416,-2.176) -- cycle;
\fill[rred] (3.613,-1.485) circle (0.208);
\draw[-{Stealth[length=1.4mm]}, rred, line width=0.7pt] (3.415,-1.637) -- (2.956,-1.992);
\path[fill=inkmuted!40] (1.962,-1.565) -- (2.082,-3.222) -- (1.668,-3.252) -- (1.548,-1.595) -- cycle;
\path[fill=inkmuted!40] (1.593,-2.216) -- (1.623,-2.631) -- (0.381,-2.721) -- (0.351,-2.307) -- cycle;
\path[fill=inkmuted!80] (1.705,-1.361) -- (1.992,-2.996) -- (1.583,-3.068) -- (1.296,-1.432) -- cycle;
\path[fill=inkmuted!80] (1.404,-2.046) -- (1.475,-2.455) -- (0.249,-2.670) -- (0.177,-2.261) -- cycle;
\draw[slotorange, line width=1.0pt, dashed] (2.187,-1.547) -- (1.882,-3.179) -- (1.474,-3.103) -- (1.588,-2.491) -- (0.364,-2.262) -- (0.440,-1.854) -- (1.664,-2.083) -- (1.779,-1.470) -- cycle;
\draw[slotblue, line width=1.0pt] (2.058,-1.622) -- (2.010,-3.282) -- (1.595,-3.270) -- (1.613,-2.647) -- (0.368,-2.611) -- (0.380,-2.196) -- (1.625,-2.232) -- (1.643,-1.610) -- cycle;
\draw[slotblue, line width=1.0pt, densely dotted] (1.652,-1.387) -- (1.915,-3.027) -- (1.505,-3.093) -- (1.406,-2.478) -- (0.176,-2.675) -- (0.111,-2.265) -- (1.340,-2.068) -- (1.242,-1.453) -- cycle;
\node[anchor=north, align=center, text=inkmuted] at (2.000,-4.050) {\textbf{Object 1}\\{\color{slotblue}\method{} 1 / 5}\quad {\color{slotorange}no memory 14 / 14}};
\fill[white] (4.320,-4.000) rectangle (8.320,0.000);
\draw[inkmuted!60, line width=0.3pt] (4.320,-4.000) rectangle (8.320,0.000);
\path[fill=inkmuted!10] (7.718,-2.889) -- (5.771,-3.158) -- (5.704,-2.671) -- (7.651,-2.403) -- cycle;
\path[fill=inkmuted!10] (6.921,-2.503) -- (6.434,-2.570) -- (6.233,-1.110) -- (6.720,-1.043) -- cycle;
\fill[rred] (7.530,-3.541) circle (0.246);
\draw[-{Stealth[length=1.4mm]}, rred, line width=0.7pt] (7.472,-3.252) -- (7.338,-2.577);
\path[fill=inkmuted!40] (7.568,-0.744) -- (6.838,-2.569) -- (6.382,-2.387) -- (7.111,-0.561) -- cycle;
\path[fill=inkmuted!40] (6.838,-1.246) -- (6.655,-1.702) -- (5.286,-1.155) -- (5.469,-0.699) -- cycle;
\path[fill=inkmuted!80] (7.292,-0.915) -- (6.107,-2.483) -- (5.715,-2.187) -- (6.900,-0.618) -- cycle;
\path[fill=inkmuted!80] (6.456,-1.207) -- (6.159,-1.599) -- (4.983,-0.710) -- (5.279,-0.318) -- cycle;
\draw[slotorange, line width=1.0pt, dashed] (7.397,-0.834) -- (6.228,-2.415) -- (5.833,-2.122) -- (6.271,-1.530) -- (5.086,-0.653) -- (5.378,-0.258) -- (6.563,-1.134) -- (7.002,-0.542) -- cycle;
\draw[slotblue, line width=1.0pt] (7.409,-0.484) -- (6.793,-2.351) -- (6.327,-2.197) -- (6.558,-1.497) -- (5.158,-1.035) -- (5.311,-0.568) -- (6.712,-1.030) -- (6.943,-0.330) -- cycle;
\draw[slotblue, line width=1.0pt, densely dotted] (7.092,-0.721) -- (5.913,-2.293) -- (5.520,-1.998) -- (5.962,-1.409) -- (4.782,-0.524) -- (5.077,-0.131) -- (6.257,-1.016) -- (6.699,-0.426) -- cycle;
\node[anchor=north, align=center, text=inkmuted] at (6.320,-4.050) {\textbf{Object 2}\\{\color{slotblue}\method{} 14 / 17}\quad {\color{slotorange}no memory 22 / 8}};
\fill[white] (8.640,-4.000) rectangle (12.640,0.000);
\draw[inkmuted!60, line width=0.3pt] (8.640,-4.000) rectangle (12.640,0.000);
\path[fill=inkmuted!10] (9.496,-2.552) -- (9.808,-0.536) -- (10.312,-0.614) -- (10.000,-2.630) -- cycle;
\path[fill=inkmuted!10] (10.117,-1.874) -- (10.195,-1.370) -- (11.707,-1.604) -- (11.629,-2.108) -- cycle;
\fill[rred] (9.116,-0.807) circle (0.255);
\draw[-{Stealth[length=1.4mm]}, rred, line width=0.7pt] (9.419,-0.849) -- (10.126,-0.948);
\path[fill=inkmuted!40] (9.819,-1.560) -- (11.626,-0.614) -- (11.863,-1.066) -- (10.056,-2.012) -- cycle;
\path[fill=inkmuted!40] (10.733,-1.657) -- (11.185,-1.421) -- (11.895,-2.776) -- (11.443,-3.012) -- cycle;
\path[fill=inkmuted!80] (9.924,-2.195) -- (11.440,-0.830) -- (11.781,-1.209) -- (10.265,-2.574) -- cycle;
\path[fill=inkmuted!80] (10.834,-2.062) -- (11.213,-1.721) -- (12.236,-2.857) -- (11.858,-3.199) -- cycle;
\draw[slotorange, line width=1.0pt, dashed] (10.282,-2.032) -- (11.942,-0.848) -- (12.238,-1.263) -- (11.616,-1.707) -- (12.504,-2.952) -- (12.089,-3.248) -- (11.201,-2.003) -- (10.578,-2.447) -- cycle;
\draw[slotblue, line width=1.0pt] (10.320,-1.160) -- (12.340,-0.875) -- (12.411,-1.380) -- (11.654,-1.487) -- (11.868,-3.001) -- (11.363,-3.073) -- (11.149,-1.558) -- (10.391,-1.665) -- cycle;
\draw[slotblue, line width=1.0pt, densely dotted] (10.278,-2.250) -- (11.981,-1.129) -- (12.262,-1.555) -- (11.623,-1.975) -- (12.464,-3.253) -- (12.038,-3.533) -- (11.197,-2.255) -- (10.558,-2.676) -- cycle;
\node[anchor=north, align=center, text=inkmuted] at (10.640,-4.050) {\textbf{Object 3}\\{\color{slotblue}\method{} 21 / 26}\quad {\color{slotorange}no memory 33 / 21}};
\fill[white] (12.960,-4.000) rectangle (16.960,0.000);
\draw[inkmuted!60, line width=0.3pt] (12.960,-4.000) rectangle (16.960,0.000);
\path[fill=inkmuted!10] (13.948,-0.949) -- (15.427,-0.692) -- (15.491,-1.062) -- (14.012,-1.319) -- cycle;
\path[fill=inkmuted!10] (14.567,-1.222) -- (14.936,-1.158) -- (15.129,-2.267) -- (14.759,-2.332) -- cycle;
\fill[rred] (15.050,-0.350) circle (0.188);
\draw[-{Stealth[length=1.4mm]}, rred, line width=0.7pt] (15.089,-0.572) -- (15.180,-1.089);
\path[fill=inkmuted!40] (14.700,-1.961) -- (15.972,-2.757) -- (15.773,-3.075) -- (14.501,-2.280) -- cycle;
\path[fill=inkmuted!40] (14.978,-2.578) -- (15.296,-2.777) -- (14.699,-3.731) -- (14.381,-3.532) -- cycle;
\path[fill=inkmuted!80] (14.262,-2.154) -- (15.648,-2.731) -- (15.503,-3.077) -- (14.118,-2.500) -- cycle;
\path[fill=inkmuted!80] (14.637,-2.716) -- (14.984,-2.861) -- (14.551,-3.900) -- (14.205,-3.756) -- cycle;
\draw[slotorange, line width=1.0pt, dashed] (14.543,-1.674) -- (15.754,-2.560) -- (15.533,-2.863) -- (15.079,-2.531) -- (14.414,-3.440) -- (14.111,-3.218) -- (14.776,-2.309) -- (14.321,-1.977) -- cycle;
\draw[slotblue, line width=1.0pt] (14.549,-1.879) -- (15.858,-2.613) -- (15.674,-2.940) -- (15.183,-2.665) -- (14.632,-3.647) -- (14.305,-3.463) -- (14.856,-2.481) -- (14.365,-2.206) -- cycle;
\draw[slotblue, line width=1.0pt, densely dotted] (14.208,-2.095) -- (15.574,-2.718) -- (15.418,-3.059) -- (14.906,-2.825) -- (14.439,-3.850) -- (14.098,-3.694) -- (14.565,-2.670) -- (14.053,-2.436) -- cycle;
\node[anchor=north, align=center, text=inkmuted] at (14.960,-4.050) {\textbf{Object 4}\\{\color{slotblue}\method{} 12 / 7}\quad {\color{slotorange}no memory 28 / 32}};
\end{tikzpicture}
\caption{Predicted 1-s poses on the first four confirmation objects (seed order, no selection; one model per arm, seed 1101, final checkpoint). Gray: the true pose of the frictionless (light) and rough (dark) version of the same object after the same push; faint: start. Orange dashed: the no-memory predictor, which must predict one pose for both versions because its inputs are identical. Blue: \method{}, solid for the frictionless and dotted for the rough version. Red disk and arrow: pusher and push direction. Below each panel: 1-s position errors in pixels (frictionless / rough).}
\label{fig:qual}
\end{figure*}

\subsection{Where the gain comes from: the ladder}
\label{sec:ladderres}

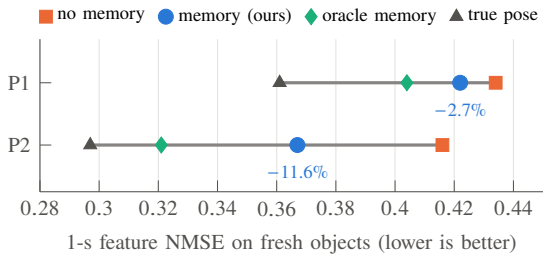
\begin{figure}[t]
\centering
\begin{tikzpicture}
\begin{axis}[
  width=0.93\columnwidth, height=36mm,
  xmin=0.28, xmax=0.45, ymin=-0.7, ymax=1.75,
  ytick={0,1}, yticklabels={{P1},{P2}},
  xlabel={1-s feature NMSE on fresh objects (lower is better)},
  xlabel style={font=\footnotesize, text=inkmuted},
  tick label style={font=\footnotesize, text=inkmuted},
  axis line style={inkmuted}, tick style={inkmuted},
  xmajorgrids, grid style={gridgray},
  axis x line*=bottom, axis y line*=left, y dir=reverse,
  clip=false,
  legend style={font=\scriptsize, draw=none, fill=none, legend columns=4, at={(0.5,1.02)}, anchor=south, /tikz/every even column/.append style={column sep=3pt}}]
\addplot[inkmuted!70, line width=1.2pt, forget plot] coordinates {(0.434,0) (0.361,0)};
\addplot[inkmuted!70, line width=1.2pt, forget plot] coordinates {(0.416,1) (0.297,1)};
\addplot[only marks, mark=square*, mark size=2.4pt, color=slotorange] coordinates {(0.434,0) (0.416,1)};
\addlegendentry{no memory}
\addplot[only marks, mark=*, mark size=2.8pt, color=slotblue] coordinates {(0.422,0) (0.367,1)};
\addlegendentry{memory (ours)}
\addplot[only marks, mark=diamond*, mark size=3pt, color=slotaqua] coordinates {(0.404,0) (0.321,1)};
\addlegendentry{oracle memory}
\addplot[only marks, mark=triangle*, mark size=2.8pt, color=inkmuted] coordinates {(0.361,0) (0.297,1)};
\addlegendentry{true pose}
\node[font=\scriptsize, text=slotblue, anchor=north] at (axis cs:0.422,0.18) {$-$2.7\pct};
\node[font=\scriptsize, text=slotblue, anchor=north] at (axis cs:0.367,1.18) {$-$11.6\pct};
\end{axis}
\end{tikzpicture}
\caption{Measured information ladder on the 512 fresh objects (post hoc; selected checkpoint; mean over 8 seeds). P1: pose tokens; P2: pose tokens plus the Fourier pose head. Each row feeds no memory, the memory, the true parameters (oracle) or the true future pose through the same pathway. The Fourier pose head lengthens the whole ladder, and the memory's gain grows with it: the memory collects 39\pct{} of the oracle's gain on P1 and 51\pct{} on P2.}
\label{fig:ladder}
\end{figure}

\begin{table}[t]
\caption{The ladder on the fresh objects, rung by rung (post hoc; 1-s NMSE at $\mu=0.35$; selected / final checkpoint; 8 seeds). Percentages are relative to no memory on the same pathway. Every rung improves on no memory for all 8 seeds, with a 95\pct{} interval over seeds and objects below zero. Ratios compare the absolute gains of P2 and P1, with two-way bootstrap 95\pct{} intervals over seeds and objects.}
\label{tab:ladder}
\centering
\setlength{\tabcolsep}{3.5pt}
\footnotesize
\begin{tabular}{@{}lcc@{}}
\toprule
Rung & P1: pose tokens & P2: + Fourier head\\
\midrule
No memory & 0.434 / 0.437 & 0.416 / 0.417\\
Memory (ours) & \ensuremath{-}2.7\pct{} / \ensuremath{-}5.0\pct & \ensuremath{-}11.6\pct{} / \ensuremath{-}12.8\pct\\
Oracle memory & \ensuremath{-}6.9\pct{} / \ensuremath{-}7.8\pct & \ensuremath{-}22.9\pct{} / \ensuremath{-}22.9\pct\\
True future pose & \ensuremath{-}16.8\pct{} / \ensuremath{-}17.5\pct & \ensuremath{-}28.5\pct{} / \ensuremath{-}28.2\pct\\
\midrule
Memory share of oracle gain & 39\pct{} / 64\pct & 51\pct{} / 56\pct\\
\midrule
Ratio P2/P1, oracle gain & \multicolumn{2}{c}{3.2 [2.5, 4.2] / 2.8 [2.3, 3.7]}\\
Ratio P2/P1, memory gain & \multicolumn{2}{c}{4.1 [2.4, 7.5] / 2.4 [1.8, 3.5]}\\
\bottomrule
\end{tabular}
\end{table}

Fig.~\ref{fig:ladder} and Table~\ref{tab:ladder} measure the ladder on the fresh objects. With pose tokens alone (P1), even the true future pose improves the forecast by only 16.8\pct{}, a perfect memory by 6.9\pct{}, and our memory by 2.7\pct{} (selected checkpoint). Yet on the development set a ridge regression on the true pose reaches 0.218: the features are highly pose-predictable, and P1 uses little of that. Adding the Fourier pose head (P2) lengthens every rung. The true-pose gain grows to 28.5\pct{}, the oracle-memory gain to 22.9\pct{} and our memory's gain to 11.6\pct{}. In absolute NMSE, the Fourier head multiplies the oracle gain by 3.2 and 2.8 at the two checkpoints (3.3 and 2.9 in relative terms) and the memory gain by 4.1 and 2.4. On the development set, where the Fourier head was chosen, the oracle ratio was 3.5--3.8$\times$; on the out-of-range objects it is 3.8--4.8$\times$. The memory's share of the oracle gain stays in a similar range as the pathway grows: it rises from 39\pct{} to 51\pct{} at the selected checkpoint and falls from 64\pct{} to 56\pct{} at the final one. In the terms of Section~\ref{sec:ladder}, the pathway set how large the measured oracle gain was, and the memory collected a similar fraction of it. The Fourier head also slightly helps the no-memory arm (0.434 to 0.416), which is why every comparison in this paper is made against a no-memory predictor with the \emph{same} pathway.

\subsection{What carries the gain: the trials, not their summary}
\label{sec:sameinfo}

The memory is a hand-designed Bayesian summary of the trials. To test whether this form matters, we gave the same motion head the 23 normalized measurements from which the memory is computed: for each push its lever arm and speed, the fitted normal, tangential and angular response and two fit residuals, plus the two noise estimates. The memory inputs were zero and, as in every arm, the predictor itself received no experience, so the measurements reached the forecast only through the motion head, exactly as the memory does; 8 seeds were trained with the sealed recipe (post hoc). The raw measurements give a similar gain. They lower the primary error by 12.5\pct{} and 13.0\pct{} relative to no memory, for all 8 seeds, and the memory differs from them by only +0.003 [\ensuremath{-}0.005, 0.012] and +0.001 [\ensuremath{-}0.008, 0.010] (Table~\ref{tab:confirm}). Across all 36 calibrated cells (two cohorts; $\mu=0$, $\mu=0.35$ and both pooled; three horizons; two checkpoints), no difference exceeds 0.009 in magnitude; one interval, for OOD objects at the final checkpoint, marginally excludes zero in favour of the raw measurements. Without a pre-specified equivalence margin, this shows no clear difference rather than equivalence. The gain of \method{} therefore does not require the Bayesian form of the summary; it comes from the trials and the pathway that delivers them. The memory keeps other advantages. It condenses the trials into eight physically meaningful numbers with an uncertainty, its size does not grow with the number of pushes, and it fits its training objects less tightly (4 of 8 seeds with a train/validation error ratio below 0.7, against 7 of 8 for the raw measurements). On pose the two trade off: the memory predicts the 1-s position better by 1.1 and 1.6\,px (the interval excludes zero at the final checkpoint), the raw measurements the heading by 0.025 and 0.016\,rad (at the selected checkpoint).

\subsection{When the gain appears: after contact}

\begin{figure}[t]
\centering
\begin{tikzpicture}
\begin{axis}[width=0.95\columnwidth, height=46mm, xmin=0, xmax=1.1, ymin=-17, ymax=9,
  xtick={0.2,0.5,1.0}, xlabel={forecast horizon (s)}, ylabel={error change (\%)},
  label style={font=\footnotesize, text=inkmuted}, tick label style={font=\footnotesize, text=inkmuted},
  axis line style={inkmuted}, tick style={inkmuted}, ymajorgrids, grid style={gridgray},
  axis x line*=bottom, axis y line*=left, clip=false,
  legend style={font=\scriptsize, draw=none, fill=none, at={(0.98,0.98)}, anchor=north east}]
\fill[inkmuted, opacity=0.12] (axis cs:0.12,-17) rectangle (axis cs:0.30,9);
\node[font=\scriptsize, text=inkmuted, anchor=south] at (axis cs:0.21,9) {first contact};
\addplot[inkmuted, dashed, line width=0.6pt, domain=0:1.1, samples=2, forget plot] {0};
\addplot[only marks, mark=*, mark size=2.2pt, color=slotblue, error bars/.cd, y dir=both, y explicit, error bar style={line width=0.9pt}]
  table[x expr=\thisrow{t}-0.018, y=rel, y error plus expr=\thisrow{hi}-\thisrow{rel}, y error minus expr=\thisrow{rel}-\thisrow{lo}, restrict expr to domain={\thisrow{cohort}}{0:0}] {data_time.dat};
\addlegendentry{in-distribution}
\addplot[only marks, mark=*, mark size=2.2pt, color=slotblue, fill=white, mark options={fill=white, line width=0.9pt}, error bars/.cd, y dir=both, y explicit, error bar style={line width=0.9pt}]
  table[x expr=\thisrow{t}+0.018, y=rel, y error plus expr=\thisrow{hi}-\thisrow{rel}, y error minus expr=\thisrow{rel}-\thisrow{lo}, restrict expr to domain={\thisrow{cohort}}{1:1}] {data_time.dat};
\addlegendentry{out-of-distribution}
\end{axis}
\end{tikzpicture}
\caption{The memory's gain emerges after contact. Relative change of the calibrated feature error with \method{} versus no memory ($\mu=0.35$, selected checkpoint; 95\pct{} intervals over seeds and objects). The gray band marks where 90\pct{} of pushes first touch the object.}
\label{fig:time}
\end{figure}
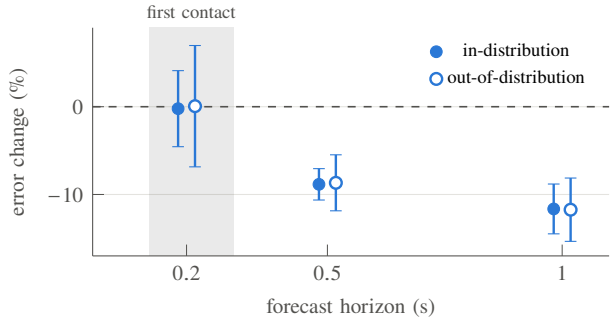

Physics can matter only once the pusher has touched the object (Fig.~\ref{fig:time}). At 0.2\,s, contact has just begun: 62\pct{} of the pushes have touched the object (median first contact 0.18\,s), and the memory changes the error by \ensuremath{-}0.2\pct{}, with an interval that includes zero. At 0.5\,s it lowers the error by 8.8\pct{}, and at 1\,s by 11.6--12.8\pct{}. The gain thus does not appear uniformly across horizons; it grows as the hidden physics shapes the motion after contact, although this pattern alone does not exclude other contributions at the later horizons. Absolute errors are not comparable across horizons: the 0.2-s targets vary more (a constant predictor scores 0.71, 0.43 and 0.58 at 0.2, 0.5 and 1\,s on validation objects), so the 0.2-s row of Table~\ref{tab:confirm} has a higher NMSE although it is relatively easier.

\subsection{Data scale}

\begin{table}[t]
\caption{Learning curve of the baseline predictor (selected checkpoint; validation set; mean over seeds). The 1-s error is calibrated, on $\mu=0.35$; a constant predictor scores 0.58.}
\label{tab:scale}
\centering
\footnotesize
\begin{tabular}{@{}lcccc@{}}
\toprule
Training objects & 64 & 256 & 1,024 & 2,048\\
\midrule
Validation NMSE (all horizons) & 0.644 & 0.508 & 0.413 & 0.385\\
Train/validation error ratio & 0.30 & 0.44 & 0.68 & 0.82\\
1-s NMSE, $\mu=0.35$ & 0.516 & 0.461 & 0.436 & 0.411\\
\bottomrule
\end{tabular}
\end{table}

With 64 training objects the predictor memorizes them (Table~\ref{tab:scale}, train/validation ratio 0.30). In an earlier 64-object pilot, which also differed in that both arms saw the raw calibration clips, even true physical parameters fed through the motion pathway changed the 1-s error by only +0.001 [\ensuremath{-}0.001, +0.003], although they cut the position error by about 21\pct{}. With 2,048 objects the predictor generalizes (ratio 0.82), and the same oracle test gives \ensuremath{-}0.021 [\ensuremath{-}0.026, \ensuremath{-}0.016]. Because the pilot changed two things at once, we read this as suggestive only: scaling the cheap simulated data was at least part of what let physical knowledge reach the features.

\section{Discussion and Limitations}
\label{sec:limits}

\textbf{What the result does and does not show.} Interaction experience, delivered through a capable pose pathway, makes a predictor in frozen V-JEPA~2 feature space forecast pushes measurably better. The improvement is confirmed on fresh objects and seeds, transfers to out-of-range physics (least clearly for low inertia) and emerges after contact. Our primary comparison isolates the experience: same architecture, memory set to zero. The comparison with the plain predictor measures the whole extension, so it should not be credited to the memory alone. The same-information control shows that the Bayesian memory is one good encoding of the experience, not a necessary one: our evidence supports trying plus a capable pathway, not the superiority of the Bayesian summary over the raw measurements.

\textbf{Scope.} The task is a simulated planar push with a known object shape, and the renderer shows only geometry, so the targets are close to pose-determined; a simulator with known geometry would forecast such scenes well. Latent world models matter where no such simulator exists, and the ladder recipe applies there as well. Whether the gain carries over to real pushing videos of objects with varied mass distributions~\cite{bauza2019omnipush} and to 3D scenes whose hidden properties must be inferred from observed interactions~\cite{tung2023physionpp} remains open. The memory uses the simulator's contact-model class and a friction prior matched to our design. The predictor is a compact model trained on frozen V-JEPA~2 features, not the released V-JEPA~2-AC model, and it still improves with more data (Table~\ref{tab:scale}). We did not test planning with the improved model. Arms that receive physical information through the Fourier pathway fit their training objects more closely: a train/validation error ratio below 0.7 occurs for 4 of 8 memory seeds, 7 of 8 raw-measurement seeds and 5 of 8 oracle seeds, against 0 of 8 for the no-memory arm. The gains hold on unseen objects regardless.

\textbf{A practical recipe.} The information ladder gives a cheap diagnostic before building any memory. Feed the true future pose and the true parameters through the intended pathway. If they help little, the pathway or its training is the likely bottleneck, and a better memory is unlikely to pay off; such trained oracle arms are diagnostic references, not guaranteed upper bounds.

\section{Conclusion}

Pushing before predicting pays off for a latent world model, provided that the physical knowledge can reach its features. With a three-push contact memory, \method{} improved a V-JEPA~2-feature predictor's 1-s forecasts by 11.6--12.8\pct{} on fresh objects in a pre-specified held-out test. A simple information ladder showed why a standard pose-token pathway gains little: the pathway, not the memory, limited the gain, and a random-Fourier pose head roughly tripled the gain of a perfect memory on fresh objects. A same-information control found a similar gain from the raw trial measurements, so the Bayesian memory, which condenses the trials into eight interpretable numbers, is a convenient encoding rather than a requirement. Learned experience encoders, real robot data and planning with the improved predictor are natural next steps.

\section*{Acknowledgment}
We disclose that experiment code and analysis scripts were written with the help of an AI coding agent (OpenAI Codex), and that the manuscript text, figures and tables were drafted with the help of an AI assistant (Anthropic Claude). The author designed the study, directed the experiments, verified all reported numbers against the experiment records, and takes full responsibility for the content.

\bibliographystyle{IEEEtran}
\bibliography{IEEEabrv,refs}

\begin{thebibliography}{10}
\providecommand{\url}[1]{#1}
\csname url@samestyle\endcsname
\providecommand{\newblock}{\relax}
\providecommand{\bibinfo}[2]{#2}
\providecommand{\BIBentrySTDinterwordspacing}{\spaceskip=0pt\relax}
\providecommand{\BIBentryALTinterwordstretchfactor}{4}
\providecommand{\BIBentryALTinterwordspacing}{\spaceskip=\fontdimen2\font plus
\BIBentryALTinterwordstretchfactor\fontdimen3\font minus
  \fontdimen4\font\relax}
\providecommand{\BIBforeignlanguage}[2]{{%
\expandafter\ifx\csname l@#1\endcsname\relax
\typeout{** WARNING: IEEEtran.bst: No hyphenation pattern has been}%
\typeout{** loaded for the language `#1'. Using the pattern for}%
\typeout{** the default language instead.}%
\else
\language=\csname l@#1\endcsname
\fi
#2}}
\providecommand{\BIBdecl}{\relax}
\BIBdecl

\bibitem{ha2018recurrent}
D.~Ha and J.~Schmidhuber, ``Recurrent world models facilitate policy
  evolution,'' in \emph{Adv. Neural Inf. Process. Syst. (NeurIPS)}, vol.~31,
  2018, pp. 2450--2462.

\bibitem{hafner2025dreamerv3}
D.~Hafner, J.~Pasukonis, J.~Ba, and T.~Lillicrap, ``Mastering diverse control
  tasks through world models,'' \emph{Nature}, vol. 640, no. 8059, pp.
  647--653, 2025.

\bibitem{lecun2022path}
\BIBentryALTinterwordspacing
Y.~LeCun, ``A path towards autonomous machine intelligence,'' OpenReview
  preprint, version 0.9.2, Jun. 2022. [Online]. Available:
  \url{https://openreview.net/forum?id=BZ5a1r-kVsf}
\BIBentrySTDinterwordspacing

\bibitem{assran2023ijepa}
M.~Assran \emph{et~al.}, ``Self-supervised learning from images with a
  joint-embedding predictive architecture,'' in \emph{Proc. IEEE/CVF Conf.
  Comput. Vis. Pattern Recognit. (CVPR)}, 2023, pp. 15\,619--15\,629.

\bibitem{bardes2024vjepa}
\BIBentryALTinterwordspacing
A.~Bardes \emph{et~al.}, ``Revisiting feature prediction for learning visual
  representations from video,'' \emph{Trans. Mach. Learn. Res.}, 2024.
  [Online]. Available: \url{https://openreview.net/forum?id=QaCCuDfBk2}
\BIBentrySTDinterwordspacing

\bibitem{assran2025vjepa2}
M.~Assran \emph{et~al.}, ``{V-JEPA} 2: Self-supervised video models enable
  understanding, prediction and planning,'' 2025, arXiv:2506.09985.

\bibitem{garrido2025intuitive}
Q.~Garrido \emph{et~al.}, ``Intuitive physics understanding emerges from
  self-supervised pretraining on natural videos,'' 2025, arXiv:2502.11831.

\bibitem{allen2020rapid}
K.~R. Allen, K.~A. Smith, and J.~B. Tenenbaum, ``Rapid trial-and-error learning
  with simulation supports flexible tool use and physical reasoning,''
  \emph{Proc. Natl. Acad. Sci. USA}, vol. 117, no.~47, pp. 29\,302--29\,310,
  2020.

\bibitem{bohg2017interactive}
J.~Bohg \emph{et~al.}, ``Interactive perception: Leveraging action in
  perception and perception in action,'' \emph{IEEE Trans. Robot.}, vol.~33,
  no.~6, pp. 1273--1291, 2017.

\bibitem{zhou2019epi}
W.~Zhou, L.~Pinto, and A.~Gupta, ``Environment probing interaction policies,''
  in \emph{Proc. Int. Conf. Learn. Represent. (ICLR)}, 2019.

\bibitem{xu2019densephysnet}
Z.~Xu, J.~Wu, A.~Zeng, J.~B. Tenenbaum, and S.~Song, ``{DensePhysNet}: Learning
  dense physical object representations via multi-step dynamic interactions,''
  in \emph{Proc. Robot.: Sci. Syst. (RSS)}, 2019.

\bibitem{wu2015galileo}
J.~Wu, I.~Yildirim, J.~J. Lim, W.~T. Freeman, and J.~B. Tenenbaum, ``{Galileo}:
  Perceiving physical object properties by integrating a physics engine with
  deep learning,'' in \emph{Adv. Neural Inf. Process. Syst. (NeurIPS)},
  vol.~28, 2015, pp. 127--135.

\bibitem{li2025pinwm}
W.~Li \emph{et~al.}, ``{PIN-WM}: Learning physics-{IN}formed world models for
  non-prehensile manipulation,'' in \emph{Proc. Robot.: Sci. Syst. (RSS)},
  2025.

\bibitem{yu2017uposi}
W.~Yu, J.~Tan, C.~K. Liu, and G.~Turk, ``Preparing for the unknown: Learning a
  universal policy with online system identification,'' in \emph{Proc. Robot.:
  Sci. Syst. (RSS)}, 2017.

\bibitem{kumar2021rma}
A.~Kumar, Z.~Fu, D.~Pathak, and J.~Malik, ``{RMA}: Rapid motor adaptation for
  legged robots,'' in \emph{Proc. Robot.: Sci. Syst. (RSS)}, 2021.

\bibitem{lin2020evf}
Y.-C. Lin, M.~Bauza, and P.~Isola, ``Experience-embedded visual foresight,'' in
  \emph{Proc. 3rd Conf. Robot Learn. (CoRL)}, ser. PMLR, vol. 100, 2020, pp.
  1015--1024.

\bibitem{tan2026latent}
K.~Tan \emph{et~al.}, ``What can latent world models know? {P}hysical
  information in multimodal predictive representations,'' 2026,
  arXiv:2607.27017.

\bibitem{liu2026sgjepa}
A.~Z. Liu, H.~Sun, L.~Baker, R.~Balestriero, and J.~Sous, ``{Semigroup-JEPA}:
  Latent dynamics consistency for zero-shot physics generalization,'' 2026,
  arXiv:2609.10464.

\bibitem{wang2026adajepa}
Y.~Wang, O.~Bounou, Y.~LeCun, and M.~Ren, ``{AdaJEPA}: An adaptive latent world
  model,'' 2026, arXiv:2606.32026.

\bibitem{zhou2025dinowm}
G.~Zhou, H.~Pan, Y.~LeCun, and L.~Pinto, ``{DINO-WM}: World models on
  pre-trained visual features enable zero-shot planning,'' in \emph{Proc. 42nd
  Int. Conf. Mach. Learn. (ICML)}, ser. PMLR, vol. 267, 2025, pp.
  79\,115--79\,135.

\bibitem{prasanna2024crssm}
S.~Prasanna, K.~Farid, R.~Rajan, and A.~Biedenkapp, ``Dreaming of many worlds:
  Learning contextual world models aids zero-shot generalization,''
  \emph{Reinforcement Learn. J.}, vol.~3, pp. 1317--1350, 2024.

\bibitem{wu2016physics101}
J.~Wu, J.~J. Lim, H.~Zhang, J.~B. Tenenbaum, and W.~T. Freeman, ``{Physics
  101}: Learning physical object properties from unlabeled videos,'' in
  \emph{Proc. Brit. Mach. Vis. Conf. (BMVC)}, 2016, pp. 39.1--39.12.

\bibitem{battaglia2016interaction}
P.~Battaglia, R.~Pascanu, M.~Lai, D.~Jimenez~Rezende, and K.~Kavukcuoglu,
  ``Interaction networks for learning about objects, relations and physics,''
  in \emph{Adv. Neural Inf. Process. Syst. (NeurIPS)}, vol.~29, 2016, pp.
  4502--4510.

\bibitem{bear2021physion}
D.~M. Bear \emph{et~al.}, ``{Physion}: Evaluating physical prediction from
  vision in humans and machines,'' in \emph{Proc. Neural Inf. Process. Syst.
  Track Datasets Benchmarks}, vol.~1, 2021.

\bibitem{yi2020clevrer}
K.~Yi \emph{et~al.}, ``{CLEVRER}: {CoLlision} events for video {REpresentation}
  and reasoning,'' in \emph{Proc. Int. Conf. Learn. Represent. (ICLR)}, 2020.

\bibitem{chen2022comphy}
Z.~Chen \emph{et~al.}, ``{ComPhy}: Compositional physical reasoning of objects
  and events from videos,'' in \emph{Proc. Int. Conf. Learn. Represent.
  (ICLR)}, 2022.

\bibitem{tung2023physionpp}
H.-Y. Tung \emph{et~al.}, ``{Physion++}: Evaluating physical scene
  understanding that requires online inference of different physical
  properties,'' in \emph{Adv. Neural Inf. Process. Syst. (NeurIPS), Datasets
  Benchmarks Track}, vol.~36, 2023, pp. 67\,048--67\,068.

\bibitem{zhan2026idpp}
G.~Zhan, X.~Ma, W.~Xie, and A.~Zisserman, ``Inferring dynamic physical
  properties from video foundation models,'' in \emph{Proc. IEEE/CVF Conf.
  Comput. Vis. Pattern Recognit. Workshops (CVPRW)}, 2026, pp. 5214--5224.

\bibitem{mason1986pushing}
M.~T. Mason, ``Mechanics and planning of manipulator pushing operations,''
  \emph{Int. J. Robot. Res.}, vol.~5, no.~3, pp. 53--71, 1986.

\bibitem{yu2016pushing}
K.-T. Yu, M.~Bauza, N.~Fazeli, and A.~Rodriguez, ``More than a million ways to
  be pushed. {A} high-fidelity experimental dataset of planar pushing,'' in
  \emph{Proc. IEEE/RSJ Int. Conf. Intell. Robots Syst. (IROS)}, 2016, pp.
  30--37.

\bibitem{bauza2017probabilistic}
M.~Bauza and A.~Rodriguez, ``A probabilistic data-driven model for planar
  pushing,'' in \emph{Proc. IEEE Int. Conf. Robot. Autom. (ICRA)}, 2017, pp.
  3008--3015.

\bibitem{ajay2018augmenting}
A.~Ajay \emph{et~al.}, ``Augmenting physical simulators with stochastic neural
  networks: Case study of planar pushing and bouncing,'' in \emph{Proc.
  IEEE/RSJ Int. Conf. Intell. Robots Syst. (IROS)}, 2018, pp. 3066--3073.

\bibitem{florence2022ibc}
P.~Florence \emph{et~al.}, ``Implicit behavioral cloning,'' in \emph{Proc. 5th
  Conf. Robot Learn. (CoRL)}, ser. PMLR, vol. 164, 2022, pp. 158--168.

\bibitem{chi2023diffusionpolicy}
C.~Chi \emph{et~al.}, ``Diffusion policy: Visuomotor policy learning via action
  diffusion,'' in \emph{Proc. Robot.: Sci. Syst. (RSS)}, 2023.

\bibitem{doshivelez2016hipmdp}
F.~Doshi-Velez and G.~Konidaris, ``Hidden parameter {Markov} decision
  processes: A semiparametric regression approach for discovering latent task
  parametrizations,'' in \emph{Proc. Int. Joint Conf. Artif. Intell. (IJCAI)},
  2016, pp. 1432--1440.

\bibitem{nagabandi2019adapt}
A.~Nagabandi \emph{et~al.}, ``Learning to adapt in dynamic, real-world
  environments through meta-reinforcement learning,'' in \emph{Proc. Int. Conf.
  Learn. Represent. (ICLR)}, 2019.

\bibitem{levy2026simdist}
J.~Levy \emph{et~al.}, ``Simulation distillation: Pretraining world models in
  simulation for rapid real-world adaptation,'' 2026, arXiv:2603.15759.

\bibitem{garnelo2018cnp}
M.~Garnelo \emph{et~al.}, ``Conditional neural processes,'' in \emph{Proc. 35th
  Int. Conf. Mach. Learn. (ICML)}, ser. PMLR, vol.~80, 2018, pp. 1704--1713.

\bibitem{zaheer2017deepsets}
M.~Zaheer, S.~Kottur, S.~Ravanbakhsh, B.~Poczos, R.~R. Salakhutdinov, and A.~J.
  Smola, ``Deep sets,'' in \emph{Adv. Neural Inf. Process. Syst. (NeurIPS)},
  vol.~30, 2017, pp. 3391--3401.

\bibitem{bauza2019omnipush}
M.~Bauza \emph{et~al.}, ``{Omnipush}: Accurate, diverse, real-world dataset of
  pushing dynamics with {RGB-D} video,'' in \emph{Proc. IEEE/RSJ Int. Conf.
  Intell. Robots Syst. (IROS)}, 2019, pp. 4265--4272.

\bibitem{rahimi2007rff}
A.~Rahimi and B.~Recht, ``Random features for large-scale kernel machines,'' in
  \emph{Adv. Neural Inf. Process. Syst. (NeurIPS)}, vol.~20, 2007, pp.
  1177--1184.

\bibitem{tancik2020fourier}
M.~Tancik \emph{et~al.}, ``{Fourier} features let networks learn high frequency
  functions in low dimensional domains,'' in \emph{Adv. Neural Inf. Process.
  Syst. (NeurIPS)}, vol.~33, 2020, pp. 7537--7547.

\bibitem{loshchilov2019adamw}
I.~Loshchilov and F.~Hutter, ``Decoupled weight decay regularization,'' in
  \emph{Proc. Int. Conf. Learn. Represent. (ICLR)}, 2019.

\bibitem{satterthwaite1946}
F.~E. Satterthwaite, ``An approximate distribution of estimates of variance
  components,'' \emph{Biometrics Bull.}, vol.~2, no.~6, pp. 110--114, 1946.

\end{thebibliography}

\end{document}